\documentclass{article}

\PassOptionsToPackage{numbers, compress}{natbib}

\usepackage[main, final]{neurips_2026}

\usepackage[utf8]{inputenc} 
\usepackage[T1]{fontenc}    
\usepackage{hyperref}       
\usepackage{url}            
\usepackage{booktabs}       
\usepackage{amsfonts}       
\usepackage{nicefrac}       
\usepackage{microtype}      
\usepackage{xcolor}         
\usepackage{pifont}

\usepackage{amsmath, amssymb, amsthm}
\usepackage{bm}
\usepackage{bbm}
\usepackage{mathtools}
\usepackage{algorithm}
\usepackage{algpseudocode}
\usepackage{hyperref}
\usepackage{xcolor}
\usepackage{booktabs}
\usepackage{natbib}

\newtheorem{proposition}{Proposition}
\newtheorem{remark}{Remark}
\usepackage{graphicx}
\usepackage{wrapfig}
\usepackage{subcaption}

\newcommand{\R}{\mathbb{R}}

\newcommand{\ba}{\bm{a}}
\newcommand{\bmu}{\bm{\mu}}
\newcommand{\bB}{\bm{B}}
\newcommand{\bD}{\bm{D}}
\newcommand{\bI}{\bm{I}}
\newcommand{\norm}[1]{\left\| #1 \right\|}
\newcommand{\inner}[2]{\left\langle #1,\, #2 \right\rangle}
\newcommand{\ours}{MAGS}

\usepackage{xcolor}

\title{Manifold-Guided Attention Steering}

\author{\begin{tabular}{ccc} Ian Li & Kapilesh Guruprasad  & Raunak Sengupta  \\ \texttt{i6li@ucsd.edu} & \texttt{kguruprasad@ucsd.edu} & \texttt{r2sengupta@ucsd.edu} \\[2ex] Ninad Satish & Loris D'Antoni & Rose Yu \\ \texttt{nsatish@ucsd.edu} & \texttt{ldantoni@ucsd.edu} & \texttt{roseyu@ucsd.edu} \\ \end{tabular}\\[5ex] University of California, San Diego}

\begin{document}

\maketitle

\begin{abstract}
Large language models frequently produce errors in reasoning tasks despite possessing the underlying knowledge required for correct reasoning. One possible approach to improve reasoning consistency is through activation steering. However, existing activation steering approaches apply fixed, pre-computed correction vectors, ignoring where the model currently sits along its generation trajectory; the result is indiscriminate perturbation that disrupts already-correct steps as freely as erroneous ones. We propose \textbf{Ma}nifold-\textbf{G}uided Attention \textbf{S}teering (\ours{}), a trajectory-aware inference-time intervention grounded in a geometric observation: the output activations of specific attention heads diverge from a low-dimensional \emph{correctness manifold} at the point of error, and this deviation compounds through subsequent steps. For each identified attention head, we learn a low-dimensional subspace from contrastive pairs of correct and incorrect traces that capture the directions along which error behavior deviates from correct behavior. During inference, we monitor each head's proximity to this manifold and apply a targeted projection correction when deviation exceeds a learned threshold, steering the attention output back toward the correct subspace before the error propagates. \ours{} consistently outperforms both unsteered baselines and static steering approaches across benchmarks spanning mathematical reasoning (MATH-500, GSM8K), code generation (HumanEval, MBPP), and molecular generation (SMILES), suggesting that correctness manifolds are a general feature of LLM attention geometry.
 
\end{abstract}

\section{Introduction}
 
Large language models (LLMs) frequently produce reasoning errors in multi-step reasoning tasks despite possessing the underlying capability to solve them. Under repeated sampling, the same model and prompt that produce an incorrect solution will often produce a correct one \citep{chen2021codex, wang2023selfconsistency, cobbe2021gsm8k}, suggesting that the capability is present, but its reliable expression is not.  Process-level annotations further confirm that most errors arise at intermediate reasoning steps rather than from a terminal absence of knowledge \citep{lightman2023verify, uesato2022solving}. Since errors arise in the generation process rather than from missing knowledge, correcting them at inference time is a natural and practical target.

Existing activation steering methods \citep{turner2023activation, zou2023representation, panickssery2024caa, vu2025angular} apply a fixed correction vector to the residual stream at generation steps. These approaches are well-suited for persistent, global behaviors (tone, style, sentiment) but are structurally mismatched to reasoning. A reasoning trajectory may proceed correctly for many steps before committing a localized error at step $t^\ast$; applying constant corrections may corrupt the correct intermediate steps while offering no guarantee of intercepting the error.
 
We hypothesize that reasoning errors manifest as a drift in a low-dimensional subspace of individual attention heads' output space: correct and incorrect trajectories occupy geometrically separable regions, and the transition from correct to incorrect behavior follows a structured, low-rank direction.  We perform diagnostic experiments to confirm this hypothesis and in fact discover that correct and incorrect trajectories are highly separable by a low-dimensional subspace of attention-head activations. This is consistent with mechanistic interpretability findings that individual attention heads are functionally specialized~\citep{elhage2021mathematical,  wang2022interpretability}, and with the linear representation hypothesis~\citep{park2023linear}, which posits that semantically meaningful distinctions are encoded along low-dimensional linear directions. 

\begin{center}
\textit{Therefore, one should steer only when their attention outputs have drifted into the error subspace.}
\end{center}

We propose \textbf{Manifold-Guided Attention Steering (\ours{})}: an adaptive intervention that dynamically steers the attention head outputs when reasoning errors are detected. \ours{}  outperforms static steering baselines across reasoning benchmarks and molecular generation on three model families, including Llama, Gemma, and GPT-OSS.

In summary, our contributions are as follows:

\begin{enumerate}
    \item We hypothesize that reasoning errors manifest as structured drift in a low-dimensional subspace of individual attention heads' output space, and confirm this hypothesis with diagnostic experiments showing that correct and incorrect trajectories are highly separable (Section~\ref{sec:hypothesis}).

    \item We propose Manifold-Guided Attention Steering (MAGS), an adaptive inference-time mechanism that monitors attention heads for reasoning drift and applies dynamic correction only when needed (Section~\ref{sec:steer}).

    \item Empirically, MAGS consistently outperforms static steering baselines on benchmarks across three model families, by up to $10.8\%$ while incurring negligible inference overhead (Section~\ref{sec:exp}).
\end{enumerate}

\section{Related Work}
\label{sec:related}
We discuss existing inference-time steering methods and geometric interpretability work. Existing steering methods apply fixed corrections without error-detection mechanisms; existing interpretability work establishes the geometric structure that we exploit and extend for adaptive intervention. To our knowledge, \ours{} is the first method to combine per-step detection with geometry-aware, conditional correction at the attention-head level. 
\begin{figure}
    \centering
    \includegraphics[width=\linewidth]{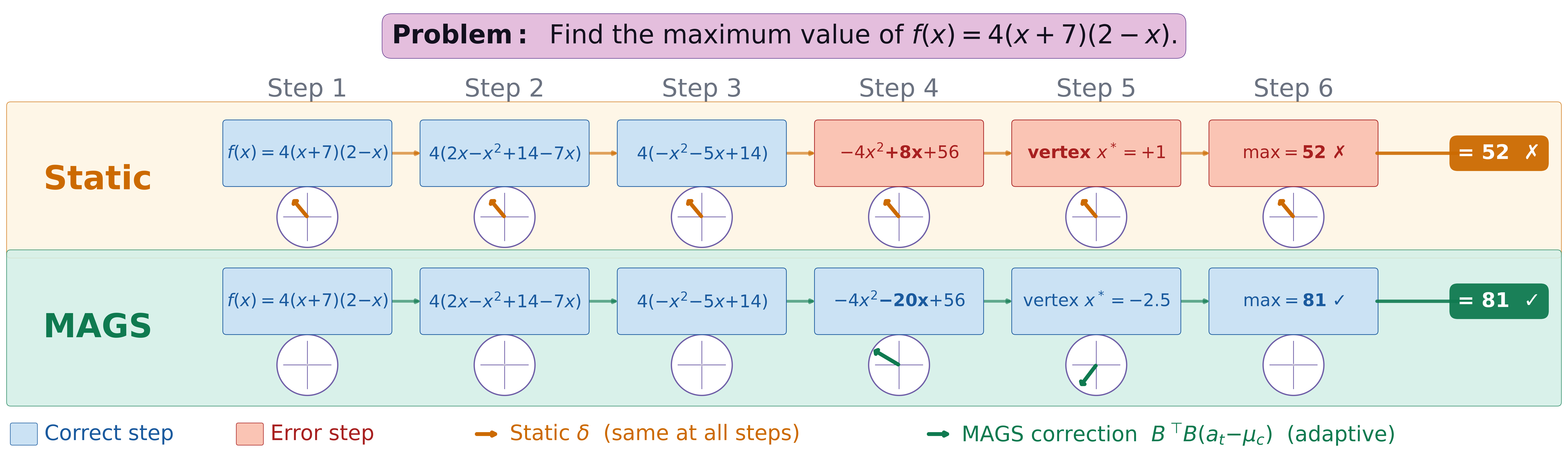}
    \caption{\textbf{Comparison of static and Manifold-Guided Attention Steering (\ours{}) on an example problem.} Step-by-step reasoning traces for a static baseline and \ours{}. Blue boxes denote correct reasoning steps; red boxes denote erroneous ones. }
    \label{fig:fig1}
\end{figure}
\subsection{Activation Steering and Inference-time Intervention}
 
Activation steering methods modify internal representations at inference time without updating model parameters. \emph{Activation Addition} \citep{turner2023activation} adds a fixed difference vector to the residual stream throughout generation. \emph{Contrastive Activation Addition}  (CAA; \citealt{panickssery2024caa}) improves reliability by averaging difference vectors across many contrastive prompt pairs. \emph{Representation  Engineering} (RepE; \citealt{zou2023representation}) extracts principal steering directions from contrastive activations via PCA. \emph{Angular  Steering} \citep{vu2025angular} replaces additive correction with a 2D rotation applied uniformly across all layers. \emph{Inference-Time Intervention} (ITI; \citealt{li2023iti}) shifts attention head outputs along a probing direction to improve truthfulness. \emph{CREST} \citep{zhang2025crest} identifies reasoning-relevant attention heads and applies fixed steering vectors to them, but does not adapt the correction to the model's current trajectory state.

All of these methods share a common limitation: they steer along a fixed direction regardless of the current activation state. MAGS addresses this by introducing a dynamic proximity trigger that fires only when a head drifts toward the error subspace and applying a step-dependent correction whose direction is determined by the current activation's projection onto the error subspace rather than a fixed vector (as illustrated in Figure~\ref{fig:fig1}).
 
\subsection{Geometric Structure of Transformer Representations}
 
A growing body of work establishes that transformer representations have rich geometric structure that can be leveraged for analysis and intervention.
 
\emph{Mechanistic interpretability} studies decompose transformer computation into interpretable circuits. \citet{elhage2021mathematical} show that attention heads implement primitive operations (copying, retrieval, inhibition) whose outputs compose additively in the residual stream.  \citet{wang2022interpretability} shows that multi-step tasks are implemented by sparse circuits across a small number of heads. Together, these results suggest that reasoning failures are likely attributable to specific heads in failure modes, which motivates our head-level intervention.
 
The \emph{linear representation hypothesis} \citep{park2023linear} posits that semantically meaningful distinctions are encoded along low-dimensional linear directions in transformer representations. \citet{burns2023ccs} show that truth has a linear representation findable by contrastive probing, establishing a precedent for our contrastive PCA construction. \citet{zou2023representation} confirm that high-level concepts, including reasoning quality, are linearly decodable from residual stream activations. MAGS extends this line of work to the per-head level, showing that the correct-to-error transition in reasoning trajectories is also linearly 
structured within individual head-output spaces.

\section{Detecting Error Drift in Attention Heads}
\label{sec:hypothesis}
 We hypothesize that incorrect reasoning traces induce drift in the output of a subset of attention heads toward a low-dimensional error subspace, geometrically separable from the subspace occupied by correct traces. We empirically validate this hypothesis by constructing a contrastive error manifold per head and demonstrating that a proximity-based score achieves strong trajectory-level error detection across layers and heads.
\subsection{Setup and Notation}

We consider a set of reasoning problems $\mathcal{P} = \{p_1, \ldots, p_N\}$. For each problem $p_i$, we generate $S$ independent reasoning traces $\mathcal{T}_i = \{\tau_{i,1}, \ldots, \tau_{i,S}\}$. Each trace $\tau$ is a token sequence of length $L_\tau$, and is assigned a binary label $y_\tau \in \{0, 1\}$ (1 = correct final answer). Assume $|\mathcal{T}_i^+| \geq 1$ and $|\mathcal{T}_i^-| \geq 1$ for all $p_i$. We write $\mathcal{T}_i^+ = \{\tau \in \mathcal{T}_i : y_\tau = 1\}$ and $\mathcal{T}_i^- = \{\tau \in \mathcal{T}_i : y_\tau = 0\}$ for the correct-trace and error-trace sets of problem $p_i$.

Given a transformer-based language model with $L$ layers and $H$ attention heads per layer, where each head operates on a $d_h$-dimensional output space. For a sample $\tau$ and head $(l,h)$, the sequence of attention head outputs is:
\begin{equation}
    \mathbf{A}_\tau^{(l,h)} \;=\;
    \bigl[\ba_1^{(l,h,\tau)},\; \ba_2^{(l,h,\tau)},\; \ldots,\;
    \ba_{L_\tau}^{(l,h,\tau)}\bigr]
    \;\in\; \R^{L_\tau \times d_h}.
\end{equation}
\subsection{Contrastive Error Manifold Construction}
\label{sec:manifold}
Given correct and incorrect reasoning traces, we construct a per-head error subspace by identifying the low-dimensional directions along which correct and incorrect activations diverge.
\paragraph{Per-problem difference vectors.}
For each problem and each head $(l, h)$, define the per-class mean:
\begin{equation}
    \bmu_{c,i}^{(l,h)} = \frac{1}{\sum_{\tau \in \mathcal{T}_i^+} L_\tau}
    \sum_{\tau \in \mathcal{T}_i^+} \sum_{t=1}^{L_\tau} \ba_t^{(l,h,\tau)},
    \qquad
    \bmu_{e,i}^{(l,h)} = \frac{1}{\sum_{\tau \in \mathcal{T}_i^-} L_\tau}
    \sum_{\tau \in \mathcal{T}_i^-} \sum_{t=1}^{L_\tau} \ba_t^{(l,h,\tau)}.
    \label{eq:mu}
\end{equation}
The \emph{contrastive difference vector} for problem $p_i$ and head $(l,h)$ is:
\begin{equation}
    \bm{\delta}_i^{(l,h)}
    \;=\;
    \bmu_{e,i}^{(l,h)} \;-\; \bmu_{c,i}^{(l,h)}
    \;\in\; \R^{d_h}.
    \label{eq:diff_vec}
\end{equation}
By construction, $\bm{\delta}_i^{(l,h)}$ cancels all directions uniformly activated by problem $p_i$ regardless of correctness, isolating the \emph{directional shift} attributable to the error. 
\paragraph{Difference matrix.}
\begin{wrapfigure}{r}{0.5\linewidth}
    \centering
    \vspace{-1.8\baselineskip}
    \includegraphics[width=\linewidth]{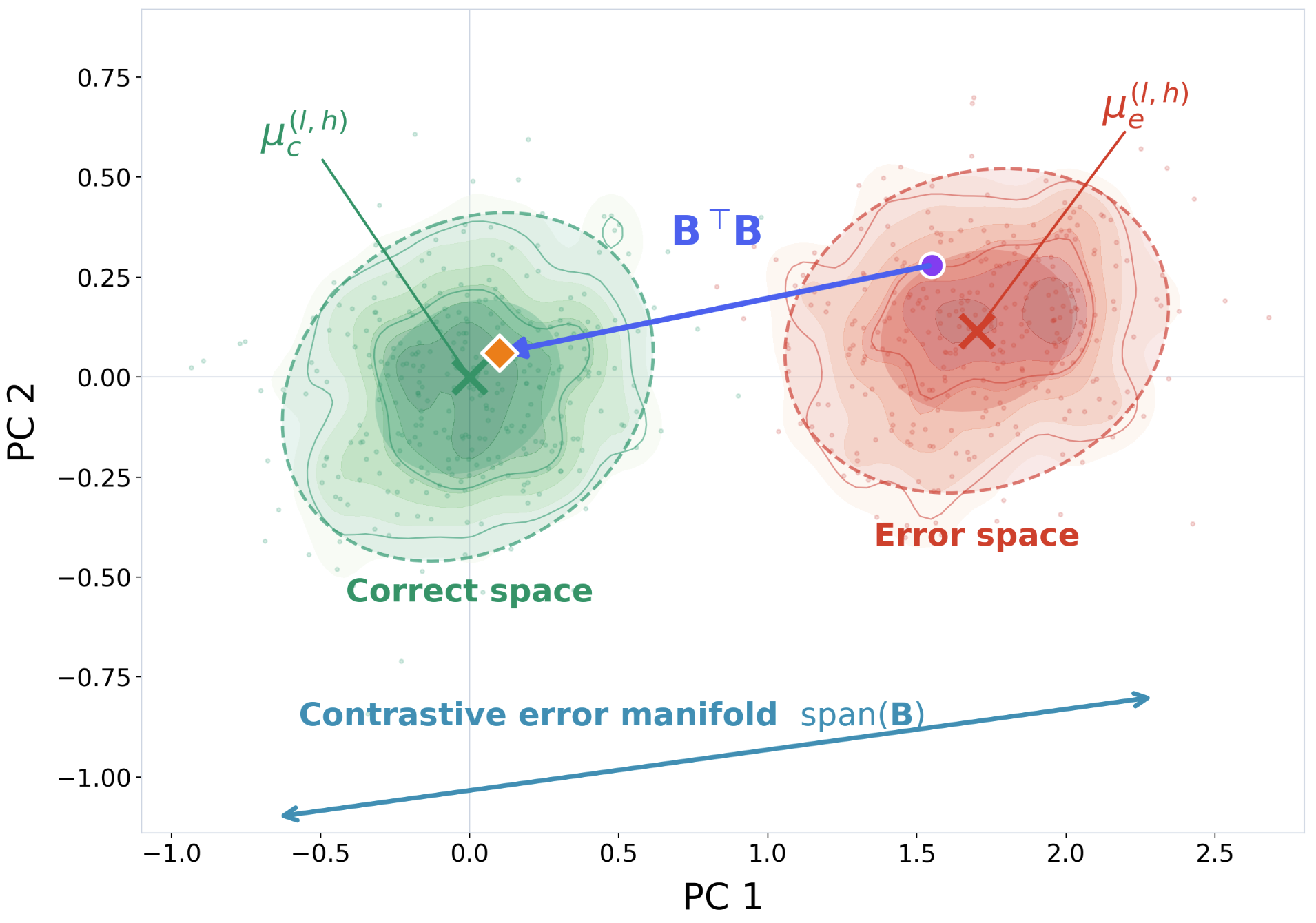}
    \vspace{-1\baselineskip}
    \caption{Schematic of the contrastive error manifold. Correct and error activation spaces are separated along the learned subspace $\mathrm{span}(\mathbf{B})$. Given an error activation $a_t$, the projection $\mathbf{B}^\top\!\mathbf{B}$ gives the direction to map  back toward $\mu_c^{(l,h)}$.}
    \label{fig:manifold}
    \vspace{-2.5\baselineskip}
\end{wrapfigure}
For a set of $N$ problems, we stack the difference vectors row-wise:
\begin{equation}
    (\mathbf{D}^{(l,h)})^T
    \;=\;
    \begin{bmatrix}
        \bm{\delta}_1^{(l,h)} & \dots & \bm{\delta}_P^{(l,h)}
    \end{bmatrix}
    \;\in\; \mathbb{R}^{d_h \times N}.
    \label{eq:diff_matrix_transpose}
\end{equation}
\paragraph{Error subspace via PCA.}
Compute the compact singular value decomposition
$\bD^{(l,h)} = \bm{U}\bm{\Sigma}\bm{V}^\top$.
Define the \emph{error subspace basis} as the top-$k$ right singular vectors:
\begin{equation}
    \bB^{(l,h)}
    \;=\;
    \bm{V}_{:,1:k}^\top
    \;\in\; \R^{k \times d_h},
    \label{eq:basis}
\end{equation}
where the rows of $\bB^{(l,h)}$ are orthonormal. The error subspace captures
the $k$ directions in head-output space along which correct-to-incorrect
deviation has the greatest variance across problems, as illustrated in
Figure~\ref{fig:manifold}. 
\vspace{0.2em}
\paragraph{Global correct-state centroid.}
Compute a global reference point from all correct traces:
\begin{equation}
    \bmu_c^{(l,h)}
    \;=\;
    \frac{\displaystyle\sum_{i}\sum_{\tau \in \mathcal{T}_i^+}
    \sum_{t=1}^{L_\tau} \ba_t^{(l,h,\tau)}}
    {\displaystyle\sum_{i}\sum_{\tau \in \mathcal{T}_i^+} L_\tau}.
    \label{eq:global_mean}
\end{equation}

This serves as the centering reference at inference time, since per-problem
means are not available during generation.
\begin{remark}
    \label{rem:contrastive}
    The choice to fit PCA on \emph{difference vectors} rather than on the raw error-state vectors is critical. If PCA were fit on $\mathcal{E}^{(l,h)}$ directly (pooled baseline), the dominant directions would reflect problem-level features (e.g., problem difficulty or topic) rather than the intrinsic error direction. The contrastive construction ensures that $\bD^{(l,h)}$ has zero mean in any direction that is uniformly correlated with problem content, leaving only the error-specific signal.
\end{remark}
\subsection{Proximity-Based Error Detection}
\label{sec:detection}

To detect when a head has drifted into the error subspace, we measure how much its current output projects onto the learned error subspace.  A large projection indicates the head is behaving similarly to how it behaves in erroneous traces. At each decode step $t$ during inference, for each monitored head $(l, h)$, we compute the \emph{proximity score}:
\begin{equation}
    d_t^{(l,h)}
    \;=\;
    \norm{\bB^{(l,h)}\bigl(\ba_t^{(l,h)} - \bmu_c^{(l,h)}\bigr)}^2
    \;=\;
    \bigl(\ba_t^{(l,h)} - \bmu_c^{(l,h)}\bigr)^\top
    \bB^{(l,h)\top}\bB^{(l,h)}
    \bigl(\ba_t^{(l,h)} - \bmu_c^{(l,h)}\bigr).
    \label{eq:proximity}
\end{equation}
This is the squared norm of the projection of the centered head output onto the error subspace. A large value indicates that the current head output has a substantial component along the learned error directions.

To avoid flagging normal generation steps, we calibrate a per-head threshold on correct traces and trigger only when the proximity score exceeds it. We fire a correction only when the proximity score exceeds a per-head threshold calibrated on correct traces,  ensuring that already-correct steps are left undisturbed. A \emph{trigger} fires at step $t$ for head $(l,h)$ when:
\begin{equation}
    d_t^{(l,h)} \;>\; \tau^{(l,h)},
    \label{eq:trigger}
\end{equation}
where $\tau^{(l,h)}$ is set to the $q$-th percentile of
$\{d_t^{(l,h)}\}$ computed over all token steps from correct trajectories in the training set. 

\begin{figure}
    \centering
    \includegraphics[width=\linewidth]{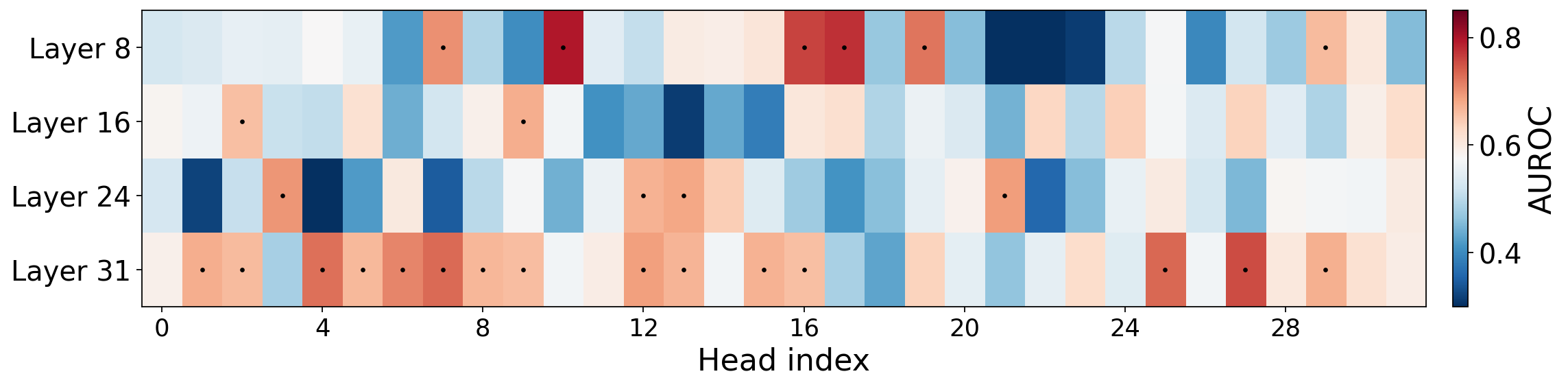}
    \caption{Per-head error detection AUROC across four monitored layers for the Math-Instruct dataset with Llama-3.1-8b-Instruct. Dots mark heads with AUROC $> 0.65$. The signal is sparse and concentrated in specific heads.}
    \label{fig:auroc_heatmap}
\end{figure}

\subsection{Empirical Validation of the Drift Hypothesis}

We validate whether the learned error subspace carries a detectable signal for distinguishing correct from incorrect reasoning trajectories. Using Math-Instruct traces from Llama-3.1-8B-Instruct collected at layers $\{8, 16, 24, 31\}$, we perform a problem-level 70/30 train/test split, ensuring that all traces of a given problem land in the same split and the manifold never observes test problems during construction. We build the contrastive error manifold on the training split and evaluate on the held-out test split.

For each trace, we aggregate the per-step proximity scores $\{d_t^{(l,h)}\}$ into a single scalar using max aggregation, and classify the trace as incorrect if the score exceeds a threshold  $\tau^{(l,h)}$ calibrated on the training set by maximizing balanced accuracy. We report trajectory-level AUROC for each $(l, h)$ pair independently. Figure~\ref{fig:auroc_heatmap} shows the per-head AUROC across all monitored layers. The signal is sparse and concentrated in specific heads (marked dots for AUROC $> 0.65$), confirming that error drift is a structured, localized phenomenon rather than a diffuse property of all heads.

\section{Manifold-Guided Attention Steering}
\label{sec:steer}
Having established that proximity scores reliably detect drift toward the error subspace, we now describe how MAGS exploits this signal to apply a targeted correction to the attention head output at inference time.

\paragraph{Head selection.}
Rather than monitoring all $L \times H$ heads, we pre-select the top-$K$ heads by \emph{held-out AUROC}: for each head, AUROC is computed between the trajectory-level error label and the mean proximity score over the trajectory, on a held-out problem split. Monitoring only the top-$K$ heads reduces per-step overhead from $O(L \cdot H \cdot k \cdot d_h)$ to $O(K \cdot k \cdot d_h)$.

\paragraph{Steering by error-component correction.} When head $(l,h)$ triggers at step $t$, we apply the following in-place correction to its output \emph{before} it is passed to the output projection of layer $l$, $W_O^{(l)}$:
\begin{equation}
    \tilde{\ba}_t^{(l,h)}
    \;=\;
    \ba_t^{(l,h)}
    \;-\;
    \alpha\,\bB^{(l,h)\top}\bB^{(l,h)}
    \bigl(\ba_t^{(l,h)} - \bmu_c^{(l,h)}\bigr),
    \label{eq:correction}
\end{equation}
where $\alpha \in (0, 1]$ is a steering strength hyperparameter that we can control empirically.

Suppose $\alpha=1$, and let $\bm{P}_\perp^{(l,h)} = \bI_{d_h} - \bB^{(l,h)\top}\bB^{(l,h)}$ denote
the orthogonal projector onto the \emph{complement} of the error subspace.
Then \eqref{eq:correction} can be written equivalently as:
\begin{equation}
    \tilde{\ba}_t^{(l,h)}
    \;=\;
    \bmu_c^{(l,h)}
    \;+\;
    \bm{P}_\perp^{(l,h)}
    \bigl(\ba_t^{(l,h)} - \bmu_c^{(l,h)}\bigr).
    \label{eq:correction_proj}
\end{equation}
This form makes the semantics transparent: we decompose the deviation of the head output from the correct-state mean into an error-subspace component and a complement component, then discard only the former. Algorithm~\ref{alg:\ours{}} summarizes the complete inference procedure,  combining the proximity check (Section~\ref{sec:detection}) and the error-component correction into a single decode loop.

We formalize the key advantage over full residual stream correction: the $d_h - k$ directions unrelated to the error manifold are completely untouched.

\begin{proposition}[Information Preservation]
    The correction \eqref{eq:correction} preserves all information in
    $\ba_t^{(l,h)}$ that lies in the $(d_h - k)$-dimensional complement of
    the error subspace.
    Specifically, for any vector $\bm{v} \in \R^{d_h}$ with
    $\bB^{(l,h)}\bm{v} = \bm{0}$:
    \begin{equation}
        \inner{\tilde{\ba}_t^{(l,h)}}{\bm{v}}
        \;=\;
        \inner{\ba_t^{(l,h)}}{\bm{v}}.
    \end{equation}
\end{proposition}

\begin{algorithm}[t]
\caption{Manifold-Guided Head Steering (\ours{}) — Inference}
\label{alg:\ours{}}
\begin{algorithmic}[1]
\Require Pre-computed manifolds $\{\bB^{(l,h)}, \bmu_c^{(l,h)}, \tau^{(l,h)}\}$ for top-$K$ heads; prompt $x_{1:\mathrm{prompt}}$
\State $t \leftarrow 0$
\While{generation not complete}
    \State $t \leftarrow t + 1$
    \State Run forward pass for token $t$; collect $\ba_t^{(l,h)}$ for all
           monitored $(l,h)$ via pre-hooks on $W_O^{(l)}$
    \For{each monitored head $(l,h)$ in layer order}
        \State Compute $d_t^{(l,h)} \leftarrow
               \norm{\bB^{(l,h)}(\ba_t^{(l,h)} - \bmu_c^{(l,h)})}^2$
               \hfill\textit{// Eq.~\eqref{eq:proximity}}
        \If{$d_t^{(l,h)} > \tau^{(l,h)}$}
            \State $\tilde{\ba}_t^{(l,h)} \leftarrow
                   \ba_t^{(l,h)} - \alpha\bB^{(l,h)\top}\bB^{(l,h)}
                   (\ba_t^{(l,h)} - \bmu_c^{(l,h)})$
                   \hfill\textit{// Eq.~\eqref{eq:correction}}
            \State Replace $\ba_t^{(l,h)}$ with $\tilde{\ba}_t^{(l,h)}$ in
                   the input to $W_O^{(l)}$
        \EndIf
    \EndFor
    \State Complete the forward pass; sample next token $x_{t+1}$
\EndWhile
\end{algorithmic}
\end{algorithm}

The overhead of \ours{} per decode step is dominated by the $K$ proximity score computations. Each requires a matrix-vector product $\bB^{(l,h)}\bm{v}$ of cost $O(k \cdot d_h)$, followed by a norm computation of cost $O(k)$. The conditional correction, when triggered, requires one additional matrix-vector product $\bB^{(l,h)\top}(\bB^{(l,h)}\bm{v})$ of cost $O(k \cdot d_h)$. The total per-step overhead is $O(K \cdot k \cdot d_h)$.

\paragraph{Compositional Steering for Multiple Objectives.}
A fundamental challenge in multi-objective steering is that naively combining steering vectors for different constraints can produce conflicting corrections. We sidestep this problem by introducing MAGS$^\text{u}$, which treats each objective independently from the start. For each objective $c_k$, MAGS$^\text{u}$ extracts a dedicated error manifold, which naturally yields disjoint head sets $\mathcal{H}_1, \dots, \mathcal{H}_K$ during head selection. Then, at inference time, MAGS$^\text{u}$ steers the \emph{union} of all selected heads, applying each head's correction through its own independently learned manifold. 

\section{Experiments}
\label{sec:exp}
\subsection{Reasoning Benchmarks}
\paragraph{Setup and metrics.}
We evaluate across two reasoning domains chosen to span difficulty and output structure: mathematical reasoning on \textbf{MATH-500} \cite{lightman2023verify} and \textbf{GSM8K} \cite{cobbe2021gsm8k}, and code generation on \textbf{HumanEval} \cite{chen2021codex} and \textbf{MBPP} \cite{austin2021program}, where correctness is verified by executing test cases rather than by string matching.  To assess generality across model families, we run all reasoning
and code experiments on two instruction-tuned language models:
\textbf{Llama-3.1-8B-Instruct}~\cite{grattafiori2024llama3} and \textbf{Gemma-4-E4b-it}~\cite{gemma4team2025}.

\paragraph{Baselines.}
We compare against two families of inference-time steering methods. 

\emph{Representation-level methods}. \textbf{Inference-Time Intervention}  (ITI;\citealt{li2023iti}), which steers individual attention head outputs
along a fixed linear probe direction; \textbf{Angular Steering} (AS; \citealt{vu2025angular}), which applies a fixed 2D rotation in the mean-difference span across all layers. Since ITI and AS were originally designed for behavioral alignment rather than reasoning, we adapt them to our setting by treating correct and incorrect solution traces as the desired and undesired contrast sets, respectively, replacing their original prompt-pair construction.

\emph{Decoding-level methods}. \textbf{Contrastive Decoding} (CD; \citealt{li2023contrastive}) contrasts  the token distributions of a large expert model and a smaller amateur model at each decoding step, providing a baseline that operates at the output distribution level and a more computationally expensive method than all steering methods.

\paragraph{Manifold construction.}
For each benchmark, we collect contrastive trace pairs (i.e., same problem, one correct solution and one incorrect) from the corresponding training split: \textbf{Math-Instruct} \cite{yue2023mammoth} for MATH-500, the \textbf{GSM8K} \cite{cobbe2021gsm8k} training set for GSM8K, and \textbf{APPS} \cite{hendrycks2021apps} for both HumanEval \cite{chen2021codex} and MBPP\cite{austin2021program}. Traces are generated by sampling the base model; problems for which both a correct and an incorrect trace cannot be obtained within 8 samples are discarded. 

\begin{table*}[ht]
\centering
\caption{Performance and output fluency (perplexity, $\downarrow$) of steering methods on
         Llama-3.1-8B-Instruct. Best accuracy per benchmark in \textbf{bold};
         best perplexity in \textit{italics}.}
\label{tab:llama_reasoning}
\resizebox{0.9\textwidth}{!}{%
\begin{tabular}{l cccc cccc}
\toprule
 & \multicolumn{2}{c}{\textbf{MATH-500}}
 & \multicolumn{2}{c}{\textbf{GSM8k}}
 & \multicolumn{2}{c}{\textbf{HumanEval}}
 & \multicolumn{2}{c}{\textbf{MBPP}} \\
\cmidrule(lr){2-3}\cmidrule(lr){4-5}\cmidrule(lr){6-7}\cmidrule(lr){8-9}
\textbf{Method} & Acc.$\uparrow$ & PPL$\downarrow$
                & Acc.$\uparrow$ & PPL$\downarrow$
                & Acc.$\uparrow$ & PPL$\downarrow$
                & Acc.$\uparrow$ & PPL$\downarrow$ \\
\midrule
Unsteered & 0.478 & 1.229
          & 0.860 & 1.196
          & 0.561 & \textit{1.121}
          & 0.562 & 2.978 \\
\midrule
ITI       & 0.498 & \textit{1.227}
          & 0.855 & 1.335
          & 0.573 & 1.136
          & 0.548 & 2.250 \\
AS        & 0.506 & 1.232
          & 0.857 & 1.335
          & 0.591 & 1.131
          & 0.546 & \textit{2.214} \\
CD        & 0.492 & 1.295
          & 0.858 & 1.271
          & 0.591 & 1.152
          & 0.555 & 2.291 \\
\midrule
\textbf{\ours{}} & \textbf{0.530} & \textit{1.227}
                 & \textbf{0.867} & \textit{1.194}
                 & \textbf{0.604} & 1.130
                 & \textbf{0.574} & 2.970 \\
\bottomrule
\end{tabular}}
\end{table*}

\begin{table*}[ht]
\centering
\caption{Performance and output fluency (perplexity, $\downarrow$) of steering methods on
         Gemma-4-E4b-it. Best accuracy per benchmark in \textbf{bold};
         best perplexity in \textit{italics}}
\label{tab:gemma_reasoning}
\resizebox{0.9\textwidth}{!}{%
\begin{tabular}{l cccc cccc}
\toprule
 & \multicolumn{2}{c}{\textbf{MATH-500}}
 & \multicolumn{2}{c}{\textbf{GSM8k}}
 & \multicolumn{2}{c}{\textbf{HumanEval}}
 & \multicolumn{2}{c}{\textbf{MBPP}} \\
\cmidrule(lr){2-3}\cmidrule(lr){4-5}\cmidrule(lr){6-7}\cmidrule(lr){8-9}
\textbf{Method} & Acc.$\uparrow$ & PPL$\downarrow$
                & Acc.$\uparrow$ & PPL$\downarrow$
                & Acc.$\uparrow$ & PPL$\downarrow$
                & Acc.$\uparrow$ & PPL$\downarrow$ \\
\midrule
Unsteered & 0.614 & \textit{1.107}
          & 0.899 & 1.710
          & 0.560 & \textit{1.231}
          & 0.587 & 5.169 \\
\midrule
ITI       & 0.646 & 1.121
          & \textbf{0.913} & 1.293
          & 0.463 & 1.267
          & 0.557 & 5.053 \\
AS        & 0.566 & 1.279
          & 0.900 & 1.416
          & 0.518 & 1.334
          & 0.548 & 5.475 \\
CD        & 0.604 & 1.130
          & 0.874 & \textit{1.107}
          & 0.433 & 1.272
          & 0.593 & \textit{4.748} \\
\midrule
\textbf{\ours{}} & \textbf{0.648} & 1.108
                 & \textbf{0.913}          & 1.210
                 & \textbf{0.604} & 1.281
                 & \textbf{0.604} & 5.059 \\
\bottomrule
\end{tabular}}
\end{table*}

  

\paragraph{Result.} As shown in Table~\ref{tab:llama_reasoning} and Table~\ref{tab:gemma_reasoning}, \ours{} consistently outperforms the unsteered baseline and all three steering methods across both models and all four benchmarks. On MATH-500, where multi-step derivations provide the most opportunities for error compounding, \ours{} achieves the largest gains: +5.2 points over the unsteered Llama baseline and +3.4  points over the unsteered Gemma baseline. On GSM8K, where problems are shorter and errors more localized, the margin over baselines narrows, consistent with the intuition that proximity-triggered correction is most valuable when correct steps substantially outnumber erroneous ones.

Code generation reveals a sharper contrast. \ours{} improves HumanEval pass@1 by approximately 4 points over the unsteered baseline on both models. All three baselines \emph{degrade} HumanEval performance relative to the unsteered Gemma model, suggesting that unconditional interventions disrupt syntactic coherence even when the model's reasoning is already correct. \ours{} avoids this failure mode since its proximity trigger suppresses corrections at steps where no drift is detected.

Beyond accuracy, we evaluate output fluency via perplexity to assess whether steering distorts the model's generation distribution. MAGS consistently matches or approaches unsteered perplexity across benchmarks, confirming that the proximity threshold suppresses corrections on already-correct steps and leaves the output distribution largely intact. ITI and Angular Steering generally raise perplexity relative to the baseline, indicating that static interventions can potentially distort the generation distribution.

\subsection{Molecular Generation}

\paragraph{Task.}
Molecular generation requires producing syntactically and chemically valid SMILES strings~\cite{weininger1988smiles}. Unlike natural language, SMILES has rigid grammar rules: a single misplaced token renders the entire molecule invalid. Beyond validity, we additionally steer toward improved binding affinity against a target protein, measured via the AutoDock-GPU docking score \cite{Santos-Martins2021}. Since an ideal molecule would be both syntactically valid and chemically strong binding, this task induces a natural multi-objective structure.
\paragraph{Setup and Metrics.}
For molecular generation, we evaluate using \textbf{GPT-OSS 20B}~\cite{openai2025gptoss}. We report \textbf{Validity} (fraction of generated SMILES parseable; higher is better) and \textbf{Binding Affinity} (measured as AutoDock-GPU scores; lower is better). We generate 500 molecules per method. Contrastive Decoding is excluded as it requires a smaller companion model; no publicly available version of GPT-OSS below 20B parameters exists at the time of writing.

\begin{table}[h]
\centering
\caption{Molecular generation by steering GPT-OSS-20B.
  Validity: higher is better (\%).
  Binding Affinity: lower score is better(kcal/mol).}
\label{tab:molgen}
\begin{tabular}{lcc}
\toprule
\textbf{Method} & \textbf{Validity ($\uparrow$)} & \textbf{Binding Affinity ($\downarrow$)} \\
\midrule
Unsteered                          & 50.4 & -7.36 \\
Angular Steering & 51.4 & -7.44 \\
ITI         & \textbf{57.8} & -7.20 \\
\textbf{\ours{} }               & \underline{54.8} & \underline{-7.54} \\
\textbf{MAGS$^\text{u}$}               & \underline{54.8} & \textbf{-7.56} \\


\bottomrule
\end{tabular}
\end{table}
\paragraph{Result.} As shown in Table~\ref{tab:molgen}, ITI illustrates the consequence of the entanglement of objectives: it substantially improves validity but incurs a significant drop in binding affinity, a sign of overcorrection toward common valid scaffolds that are grammatically safe but chemically generic. MAGS$^{\text{u}}$ avoids this interference by learning two independent manifolds: the affinity manifold (contrasting high and low binding affinity) and the validity manifold (contrasting valid and invalid molecules). At inference time, the union of both subspaces applies each correction in its own independent direction, allowing binding affinity to improve without the validity–affinity interference observed in joint methods.

\section{Discussion}
\paragraph{Ablation on hyperparameter sensitivity.}
While we detail the experiment hyperparameters in Appendix~\ref{app:ablation}, Table~\ref{tab:ablation} reveals a clear contrast in robustness across methods. Steering Llama-3.1-8B on MATH-500, MAGS accuracy varies narrowly (0.492–0.530) across all configurations tested, indicating that the proximity threshold already governs when corrections are applied and reduces sensitivity to other hyperparameters. ITI spans a much wider range (0.214–0.498), with performance collapsing at high steering strength, confirming that static interventions are brittle to miscalibration. Angular Steering shows the greatest variance (0.206–0.506), with accuracy swinging dramatically across rotation angles.

\begin{figure}
    \centering
    \includegraphics[width=\linewidth]{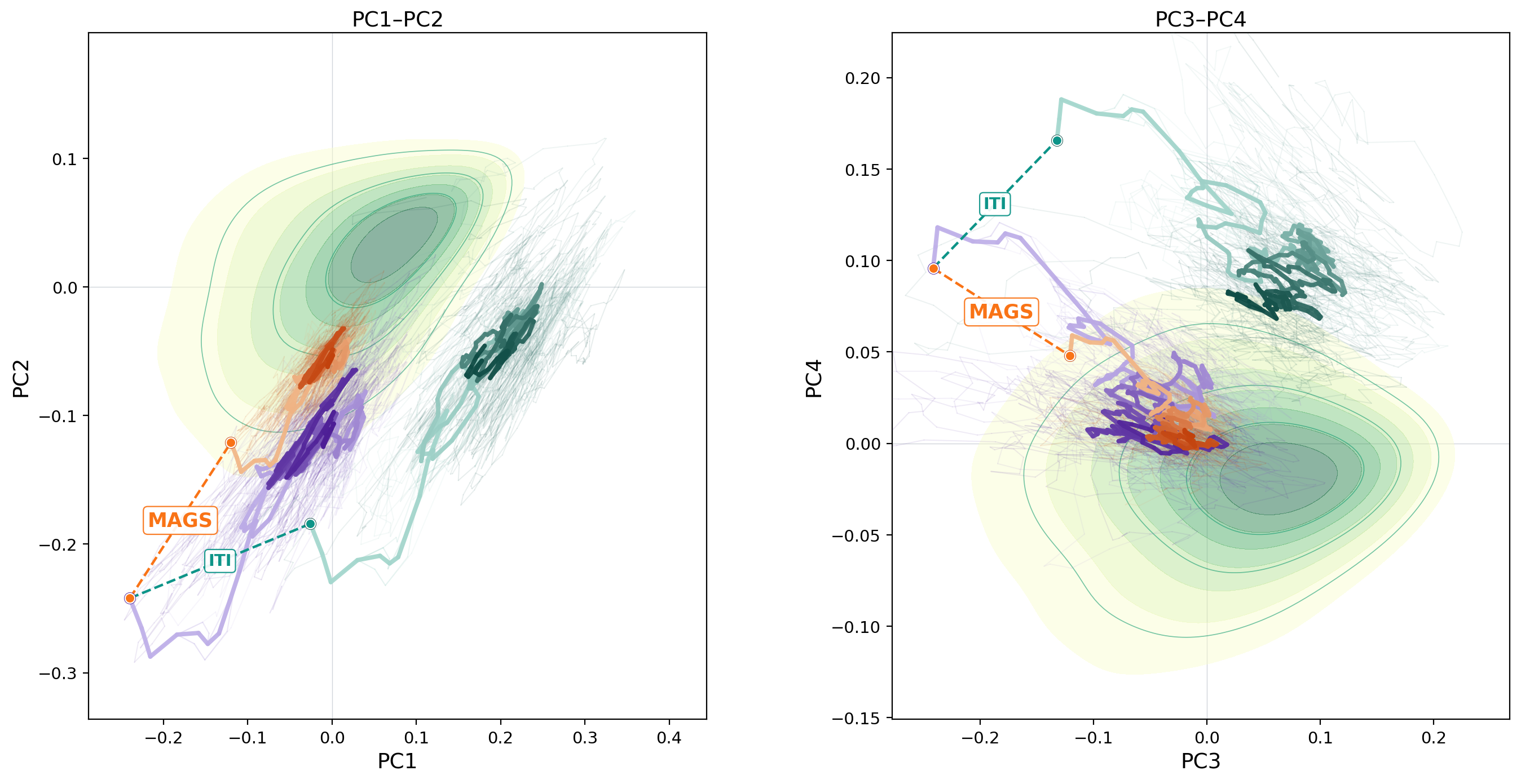}
    \caption{Latent-space trajectories of attention-head activations projected onto the top-4 principal components of the contrastive error subspace. \textbf{Filled contours}: kernel-density estimate of the correct-output activation distribution. \textbf{Purple}: unsteered mean trajectory (light $\to$ dark = early $\to$ late generation steps). \textbf{Orange}: MAGS-steered mean trajectory. \textbf{Teal}: ITI-steered mean trajectory. Faint lines show individual problem traces. Dashed connectors mark a correction step for MAGS  and ITI.}
    \label{fig:traj}
    \vspace{-1em}
\end{figure}
\paragraph{Latent trajectory analysis.}
To visualize how MAGS affects the model's internal representations during generation, we project the attention-head output activations at a steered layer onto the top-4 principal components of the contrastive error subspace. We compare MAGS against ITI as ITI also operates directly on attention-head representations. We run parallel steered and unsteered decodes on the same set of problems, collect per-step activations, and plot the resulting trajectories in the PC1-PC2 and PC3-PC4 planes.

Figure~\ref{fig:traj} reveals a clear divergence between methods. The unsteered trajectory (purple) drifts steadily away from the correct-output distribution and fails to recover, consistent with the hypothesis that reasoning errors manifest as directional activation drift. MAGS (orange) intervenes at the correction step (orange dashed connector) and immediately redirects the trajectory back toward the high-density correct
region, where it remains for the rest of the generation. ITI (teal), even after intervention, continues to diverge and settles in a region far from the correct-output distribution, which geometrically explains why static interventions can degrade performance: without a subspace constraint, the correction vector pushes activations in an imprecise direction that does not align with the correct-output manifold.

\section{Conclusion}
\label{sec:conclusion}
We presented Manifold-Guided Attention Steering (\ours{}), an inference-time method that corrects attention-head activations by projecting them back onto a low-rank manifold learned from correct-output traces, intervening only when a proximity threshold $\tau$ signals that the current activation has drifted clearly off-manifold and leaving already-correct steps undisturbed.  Across mathematical reasoning, code generation, and molecular generation on three model families, \ours{} consistently outperforms ITI, Angular Steering, and Contrastive Decoding. The manifold perspective further reframes steering as a \emph{drift detection} problem: the low-rank basis $B$ identifies the subspace most predictive of error onset, providing interpretable geometric structure that scalar-magnitude methods lack, as visualized through latent-space projections of steered trajectories being pulled back into the high-density region of correct-output activations. 

\paragraph{Limitations and future work.} \ours{} requires a set of contrastive-output traces for manifold fitting, which may be difficult to curate in low-resource domains. The current method also operates on a fixed set of target heads identified offline; an adaptive scheme that selects heads dynamically based on generation context could reduce reliance on this pre-selection step.  More broadly, we treat each head independently; jointly modeling interactions between heads may yield stronger corrections at lower computational cost.  Finally, while our experiments cover three domains, the degree to which the learned manifold transfers across tasks within a domain (e.g., arithmetic to symbolic reasoning) remains an open question. We leave these directions to future work.

\section*{Acknowledgements}
This work was supported, in part, by the National Science Foundation under grants CCF-2546822, CCF-2506134, CCF-2446711, CCF-2422214, \#2205093, \#2146343, \#2134274, \#2441832; the Schmidt Foundation; and a Microsoft Faculty Fellowship. This work was also supported in part by the U.S. Army Research Office under Army-ECASE award W911NF-07-R-0003-03; the U.S. Department of Energy, Office of Science; the ARPA-H SOL-24-101 program; the IARPA HAYSTAC Program; DARPA YFA; and CDC-RFA-FT-23-0069.
\bibliographystyle{plainnat}
\bibliography{refs}

@article{turner2023activation,
  title     = {Steering Language Models With Activation Engineering},
  author    = {Turner, Alex and Thiergart, Lisa and Udell, David and
               Leike, Jan and Mini, Ulisse and MacDiarmid, Monte},
  journal   = {arXiv preprint arXiv:2308.10248},
  year      = {2023}
}

@article{zou2023representation,
  title={Representation engineering: A top-down approach to ai transparency},
  author={Zou, Andy and Phan, Long and Chen, Sarah and Campbell, James and Guo, Phillip and Ren, Richard and Pan, Alexander and Yin, Xuwang and Mazeika, Mantas and Dombrowski, Ann-Kathrin and others},
  journal={arXiv preprint arXiv:2310.01405},
  year={2023}
}

@article{park2023linear,
  title     = {The Linear Representation Hypothesis and the Geometry of
               Large Language Models},
  author    = {Park, Kiho and Choe, Yo Joong and Veitch, Victor},
  journal   = {arXiv preprint arXiv:2311.03658},
  year      = {2023}
}

@article{elhage2021mathematical,
   title={A Mathematical Framework for Transformer Circuits},
   author={Elhage, Nelson and Nanda, Neel and Olsson, Catherine and Henighan, Tom and Joseph, Nicholas and Mann, Ben and Askell, Amanda and Bai, Yuntao and Chen, Anna and Conerly, Tom and DasSarma, Nova and Drain, Dawn and Ganguli, Deep and Hatfield-Dodds, Zac and Hernandez, Danny and Jones, Andy and Kernion, Jackson and Lovitt, Liane and Ndousse, Kamal and Amodei, Dario and Brown, Tom and Clark, Jack and Kaplan, Jared and McCandlish, Sam and Olah, Chris},
   year={2021},
   journal={Transformer Circuits Thread},
   note={https://transformer-circuits.pub/2021/framework/index.html}
}

@article{wang2022interpretability,
  title     = {Interpretability in the Wild: a Circuit for Indirect Object
               Identification in {GPT}-2 Small},
  author    = {Wang, Kevin and Variengien, Alexandre and Conmy, Arthur and
               Shlegeris, Buck and Steinhardt, Jacob},
  journal   = {arXiv preprint arXiv:2211.00593},
  year      = {2022}
}

@article{lightman2023verify,
  title     = {Let's Verify Step by Step},
  author    = {Lightman, Hunter and Kosaraju, Vineet and Burda, Yura and
               Edwards, Harri and Baker, Bowen and Lee, Teddy and
               Leike, Jan and Schulman, John and Sutskever, Ilya and
               Cobbe, Karl},
  journal   = {arXiv preprint arXiv:2305.20050},
  year      = {2023}
}

@article{zhang2025crest,
  title={Understanding and Steering the Cognitive Behaviors of Reasoning Models at Test-Time},
  author={Zhang, Zhenyu and Wu, Xiaoxia and Zhou, Zhongzhu and Wu, Qingyang and Zhang, Yineng and Ponnusamy, Pragaash and Subbaraj, Harikaran and Wang, Jue and Song, Shuaiwen Leon and Athiwaratkun, Ben},
  journal={arXiv preprint arXiv:2512.24574},
  year={2025}
}

@inproceedings{panickssery2024caa,
    title = "Steering Llama 2 via Contrastive Activation Addition",
    author = "Rimsky, Nina  and
      Gabrieli, Nick  and
      Schulz, Julian  and
      Tong, Meg  and
      Hubinger, Evan  and
      Turner, Alexander",
    editor = "Ku, Lun-Wei  and
      Martins, Andre  and
      Srikumar, Vivek",
    booktitle = "Proceedings of the 62nd Annual Meeting of the Association for Computational Linguistics (Volume 1: Long Papers)",
    month = aug,
    year = "2024",
    address = "Bangkok, Thailand",
    publisher = "Association for Computational Linguistics",
    url = "https://aclanthology.org/2024.acl-long.828/",
    doi = "10.18653/v1/2024.acl-long.828",
    pages = "15504--15522",
}

@inproceedings{
li2023iti,
title={Inference-Time Intervention: Eliciting Truthful Answers from a Language Model},
author={Kenneth Li and Oam Patel and Fernanda Vi{\'e}gas and Hanspeter Pfister and Martin Wattenberg},
booktitle={Thirty-seventh Conference on Neural Information Processing Systems},
year={2023}
}

@inproceedings{
vu2025angular,
title={Angular Steering: Behavior Control via Rotation in Activation Space},
author={Hieu M. Vu and Tan Minh Nguyen},
booktitle={The Thirty-ninth Annual Conference on Neural Information Processing Systems},
year={2025},
}

@article{austin2021program,
  title={Program Synthesis with Large Language Models},
  author={Austin, Jacob and Odena, Augustus and Nye, Maxwell and Bosma, Maarten and Michalewski, Henryk and Dohan, David and Jiang, Ellen and Cai, Carrie and Terry, Michael and Le, Quoc and Sutton, Charles},
  journal={arXiv preprint arXiv:2108.07732},
  year={2021}}

@article{chen2021codex,
  title={Evaluating Large Language Models Trained on Code},
  author={Mark Chen and Jerry Tworek and Heewoo Jun and Qiming Yuan and Henrique Pond{\'e} and Jared Kaplan and Harrison Edwards and Yura Burda and Nicholas Joseph and Greg Brockman and Alex Ray and Raul Puri and Gretchen Krueger and Michael Petrov and Heidy Khlaaf and Girish Sastry and Pamela Mishkin and Brooke Chan and Scott Gray and Nick Ryder and Mikhail Pavlov and Alethea Power and Lukasz Kaiser and Mo Bavarian and Clemens Winter and Phil Tillet and Felipe Petroski Such and David W. Cummings and Matthias Plappert and Fotios Chantzis and Elizabeth Barnes and Ariel Herbert-Voss and William H. Guss and Alex Nichol and Igor Babuschkin and Suchir Balaji and Shantanu Jain and Andrew Carr and Jan Leike and Josh Achiam and Vedant Misra and Evan Morikawa and Alec Radford and Matthew M. Knight and Miles Brundage and Mira Murati and Katie Mayer and Peter Welinder and Bob McGrew and Dario Amodei and Sam McCandlish and Ilya Sutskever and Wojciech Zaremba},
  journal={arXiv preprint arXiv: 2107:03374},
  year={2021}
}

@article{weininger1988smiles,
author = {Weininger, David},
title = {SMILES, a chemical language and information system. 1. Introduction to methodology and encoding rules},
journal = {Journal of Chemical Information and Computer Sciences},
volume = {28},
number = {1},
pages = {31-36},
year = {1988},
doi = {10.1021/ci00057a005},
URL = {https://doi.org/10.1021/ci00057a005},
eprint = {https://doi.org/10.1021/ci00057a005}}

@article{cobbe2021gsm8k,
  title={Training Verifiers to Solve Math Word Problems},
  author={Cobbe, Karl and Kosaraju, Vineet and Bavarian, Mohammad and Chen, Mark and Jun, Heewoo and Kaiser, Lukasz and Plappert, Matthias and Tworek, Jerry and Hilton, Jacob and Nakano, Reiichiro and Hesse, Christopher and Schulman, John},
  journal={arXiv preprint arXiv:2110.14168},
  year={2021}
}

@article{burns2023ccs,
  title={Discovering Latent Knowledge in Language Models Without Supervision},
  author={Burns, Collin and Ye, Haotian and Klein, Dan and Steinhardt, Jacob},
  journal={arXiv preprint arXiv:2212.03827},
  year={2022}
}

@inproceedings{
wang2023selfconsistency,
title={Self-Consistency Improves Chain of Thought Reasoning in Language Models},
author={Xuezhi Wang and Jason Wei and Dale Schuurmans and Quoc V Le and Ed H. Chi and Sharan Narang and Aakanksha Chowdhery and Denny Zhou},
booktitle={The Eleventh International Conference on Learning Representations },
year={2023}
}

@article{uesato2022solving,
  title={Solving Math Word Problems with Process- and Outcome-Based Feedback},
  author={Jonathan Uesato and Nate Kushman and Ramana Kumar and Francis Song and Noah Siegel and Lisa Wang and Antonia Creswell and Geoffrey Irving and Irina Higgins},
  journal={arXiv preprint arXiv:2211.14275},
  year={2022}
}

@inproceedings{li2023contrastive,
    title = "Contrastive Decoding: Open-ended Text Generation as Optimization",
    author = "Li, Xiang Lisa  and
      Holtzman, Ari  and
      Fried, Daniel  and
      Liang, Percy  and
      Eisner, Jason  and
      Hashimoto, Tatsunori  and
      Zettlemoyer, Luke  and
      Lewis, Mike",
    editor = "Rogers, Anna  and
      Boyd-Graber, Jordan  and
      Okazaki, Naoaki",
    booktitle = "Proceedings of the 61st Annual Meeting of the Association for Computational Linguistics (Volume 1: Long Papers)",
    month = jul,
    year = "2023",
    address = "Toronto, Canada",
    publisher = "Association for Computational Linguistics",
    url = "https://aclanthology.org/2023.acl-long.687/",
    doi = "10.18653/v1/2023.acl-long.687",
    pages = "12286--12312",

}

@article{yue2023mammoth,
  title={MAmmoTH: Building Math Generalist Models through Hybrid Instruction Tuning},
  author={Xiang Yue and Xingwei Qu and Ge Zhang and Yao Fu and Wenhao Huang and Huan Sun and Yu Su and Wenhu Chen},
  journal={arXiv preprint arXiv:2309.05653},
  year={2023}
}

@inproceedings{hendrycks2021apps,
  title={Measuring Coding Challenge Competence With APPS},
  author={Dan Hendrycks and Steven Basart and Saurav Kadavath and Mantas Mazeika and Akul Arora and Ethan Guo and Collin Burns and Samir Puranik and Horace He and Dawn Song and Jacob Steinhardt},
  booktitle={The Thirty-fifth Annual Conference on Neural Information Processing Systems},
  year={2021}
}

@article{grattafiori2024llama3,
  title   = {The Llama 3 Herd of Models},
  author  = {Grattafiori, Aaron and others},
  journal = {arXiv preprint arXiv:2407.21783},
  year    = {2024}
}

@misc{gemma4team2025,
  title        = {Gemma 4 Technical Report},
  author       = {{Gemma Team, Google DeepMind}},
  year         = {2025},
  howpublished = {\url{https://ai.google.dev/gemma/docs/core/model_card_4}}
}

@article{openai2025gptoss,
  title         = {gpt-oss-120b \& gpt-oss-20b Model Card},
  author        = {{OpenAI}},
  year          = {2025},
  journal        = {arXiv preprint arXiv:2508.10925}
}

@Article{Santos-Martins2021,
author={Santos-Martins, Diogo
and Solis-Vasquez, Leonardo
and Tillack, Andreas F.
and Sanner, Michel F.
and Koch, Andreas
and Forli, Stefano},
title={Accelerating AutoDock4 with GPUs and Gradient-Based Local Search},
journal={Journal of Chemical Theory and Computation},
year={2021},
month={Feb},
day={09},
publisher={American Chemical Society},
volume={17},
number={2},
pages={1060-1073},
issn={1549-9618},
doi={10.1021/acs.jctc.0c01006},
url={https://doi.org/10.1021/acs.jctc.0c01006}
}
\newpage

\appendix

\section{Proof of Proposition 1}
\begin{proof}
    Expanding using \eqref{eq:correction}:
    \begin{align}
        \inner{\tilde{\ba}_t^{(l,h)}}{\bm{v}}
        &= \inner{\ba_t^{(l,h)}}{\bm{v}}
        - \inner{\bB^{(l,h)\top}\bB^{(l,h)}
            (\ba_t^{(l,h)} - \bmu_c^{(l,h)})}{\bm{v}} \\
        &= \inner{\ba_t^{(l,h)}}{\bm{v}}
        - \inner{\bB^{(l,h)}(\ba_t^{(l,h)} - \bmu_c^{(l,h)})}
                {\bB^{(l,h)}\bm{v}} \\
        &= \inner{\ba_t^{(l,h)}}{\bm{v}}
        - \inner{\bB^{(l,h)}(\ba_t^{(l,h)} - \bmu_c^{(l,h)})}{\bm{0}}
        \;=\; \inner{\ba_t^{(l,h)}}{\bm{v}}. \qedhere
    \end{align}
\end{proof}

\section{Additional Experimental Details}
\label{app:ablation}
\subsection{Ablation on Hyperparameter Sensitivity}
For ITI, we follow the hyperparameter ranges reported in the original paper~\cite{li2023iti}, sweeping over the number of steered heads in $\{24, 48, 96\}$ and intervention strength $\alpha \in \{0.5, 1.0, 5.0\}$. For Angular Steering, rotation angles are sampled uniformly across the full $360^{\circ}$ range as evaluated in the original work~\cite{vu2025angular}. For MAGS, we select hyperparameters to minimize unnecessary interference: we restrict the number of steered heads to a small set (top-1 or top-3 by manifold signal strength) and sweep projection strength $\alpha \in \{0.3, 0.5, 0.7, 1.0\}$, reflecting the principle that corrections should be both targeted and conservative. Corrections are applied only to heads where drift is detected and only to the extent required to return activations to the correct-output manifold.

\begin{table*}[h]
\centering
\caption{Hyperparameter ablations on MATH-500 (Llama-3.1-8B-Instruct).
         Best result per method in \textbf{bold}.}
\label{tab:ablation}

\begin{subtable}[t]{0.28\linewidth}
\centering
\caption{MAGS: top-$k$ heads and steering strength $\alpha$.}
\begin{tabular}{ccr}
\toprule
Config & $\alpha$ & Acc. \\
\midrule
top-3 & 0.3 & 0.502 \\
top-3 & 0.5 & 0.498 \\
top-3 & 0.7 & 0.492 \\
top-3 & 1.0 & \textbf{0.530} \\
\midrule
top-1 & 0.7 & 0.492 \\
top-1 & 1.0 & \textbf{0.530} \\
\bottomrule
\end{tabular}
\end{subtable}
\hfill
\begin{subtable}[t]{0.36\linewidth}
\centering
\caption{ITI: number of steered heads $K$ and strength $\alpha$.}
\begin{tabular}{crrr}
\toprule
 & \multicolumn{3}{c}{Strength $\alpha$} \\
\cmidrule(lr){2-4}
$K$ & 0.5 & 1.0 & 5.0 \\
\midrule
24  & 0.476 & 0.472 & 0.298 \\
48  & 0.488 & 0.472 & 0.214 \\
96  & \textbf{0.498} & 0.480 & 0.314 \\
\bottomrule
\end{tabular}
\end{subtable}
\hfill
\begin{subtable}[t]{0.28\linewidth}
\centering
\caption{Angular Steering: rotation angle.}
\begin{tabular}{cr}
\toprule
Angle & Acc. \\
\midrule
$0^{\circ}$   & 0.488 \\
$30^{\circ}$  & \textbf{0.506} \\
$60^{\circ}$  & 0.466 \\
$90^{\circ}$  & 0.284 \\
$120^{\circ}$ & 0.206 \\
$150^{\circ}$ & 0.334 \\
$180^{\circ}$ & 0.412 \\
$210^{\circ}$ & 0.400 \\
$240^{\circ}$ & 0.344 \\
$270^{\circ}$ & 0.430 \\
$300^{\circ}$ & 0.500 \\
$330^{\circ}$ & 0.468 \\
\bottomrule
\end{tabular}
\end{subtable}

\end{table*}
\begin{figure}
    \centering
    \includegraphics[width=\linewidth]{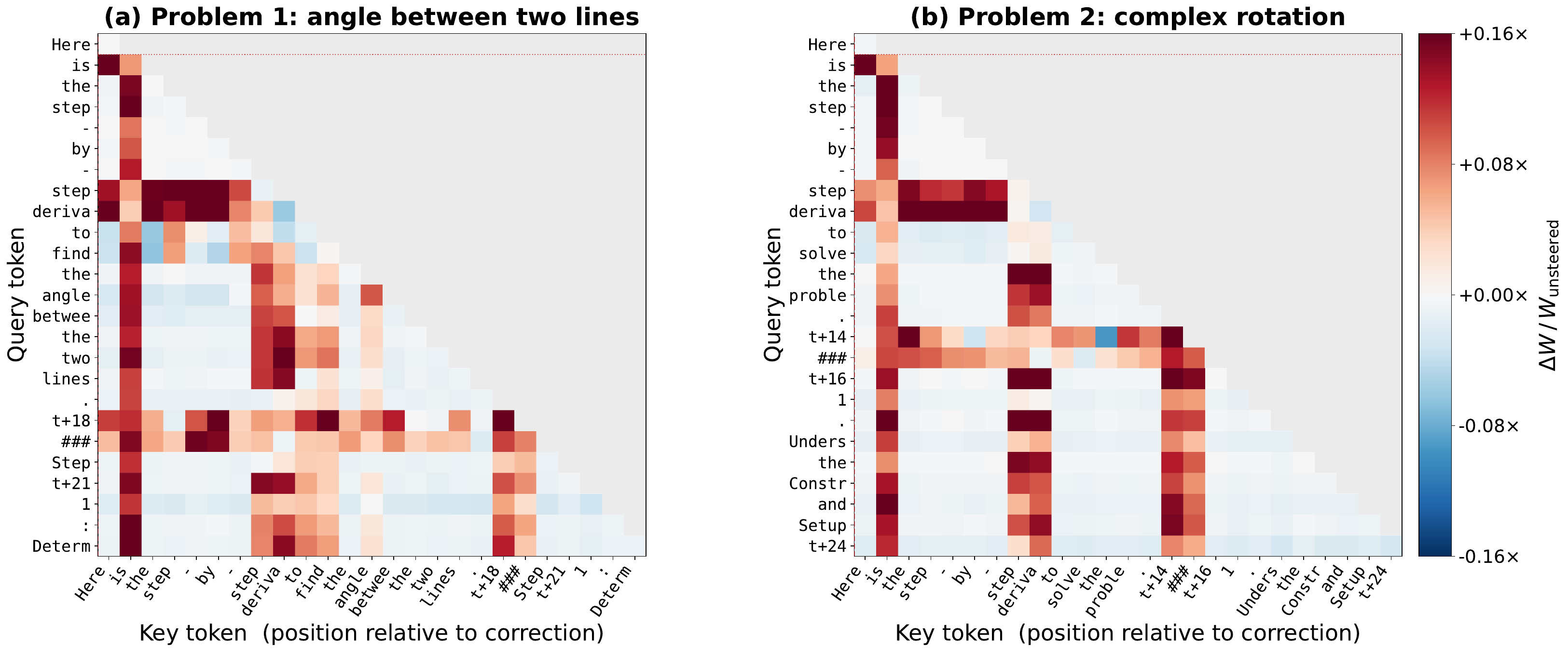}
    \caption{Relative attention shift $\Delta W / W_{\text{unsteered}}$ at layer $\ell_{\text{bip}}$ for two representative problems from MATH-500, steered with Gemma-4-E4b-it. Each cell shows the head-averaged attention change between the steered and unsteered model on the same forced token sequence. The dashed horizontal line marks the first correction step. $t_{\text{fire}}$}
    \label{fig:attn_bipartite}
\end{figure}
\newpage
\section{Visualization of the effect of steering on attention graph.} To examine how \ours{} reshapes information routing, we visualize the \emph{relative attention shift} at a layer after the correction layer, denoted~$\ell_{\mathrm{bip}}$.  Since \ours{} modifies the head \emph{output} $\mathbf{a}_h = A_h V_h$ rather than the routing matrix~$A_h$ directly, the corrected residual is written into the key-value cache at position~$t_{\mathrm{fire}}$. Then, subsequent query steps can attend to this corrected entry. To isolate the effect of accumulated KV-cache differences, both steered and unsteered streams are run on the same forced token sequence (the steered model's greedy outputs), so the two caches diverge solely through the corrections.
For each problem we average attention weights over all heads and compute
\[
  \frac{\Delta W}{W} \;=\;
  \frac{W_{\mathrm{steered}} - W_{\mathrm{unsteered}}}
       {W_{\mathrm{unsteered}} + \varepsilon}.
\]
Figure~\ref{fig:attn_bipartite} shows two representative problems from the
MATH-500 dataset for Gemma-4-E4b-it; red (blue) indicates positions attended
to more (less) by the steered model relative to the baseline.

The strongest signal in both panels is a set of \emph{vertical stripes}: certain key tokens receive consistently higher attention across all subsequent query steps. This reveals that \ours{} does not merely fix a single prediction in isolation, but makes the corrected reasoning step a more \emph{retrievable} context anchor: the entire generation following the correction keeps routing information through those tokens, reinforcing their influence persistently. The effect is consistent across both problems despite their different mathematical content, suggesting that this ``attention anchoring'' behavior may be a general mechanism by which \ours{} steers the model back onto a correct reasoning trajectory.

\subsection{Computational infrastructure.}
Experiments with Llama-3.1-8B-Instruct and Gemma-4-E4b-it were conducted on NVIDIA RTX 4090 GPUs (24\,GB VRAM) with model weights loaded in \texttt{float16} precision. GPT-OSS 20B experiments were run on NVIDIA H200 GPUs (141\,GB HBM3) due to
the larger memory footprint of the model.

\newpage
\section{Statistical Significance}
\label{app:bootstrap}

All confidence intervals are 95\% percentile bootstrap CIs with
$B{=}10000$ resamples (seed = 42).
Point estimates are original trial accuracies.
\begin{table}[h]
\centering
\caption{Bootstrap 95\% CIs on Gemma-4-E4b-it across all benchmarks
         ($B{=}10{,}000$ resamples, percentile method).}
\label{tab:bootstrap_gemma}
\resizebox{\textwidth}{!}{%
\begin{tabular}{l cc cc cc cc}
\toprule
 & \multicolumn{2}{c}{\textbf{MATH-500} ($N{=}500$)}
 & \multicolumn{2}{c}{\textbf{GSM8k} ($N{=}1319$)}
 & \multicolumn{2}{c}{\textbf{HumanEval} ($N{=}164$)}
 & \multicolumn{2}{c}{\textbf{MBPP} ($N{=}427$)} \\
\cmidrule(lr){2-3}\cmidrule(lr){4-5}\cmidrule(lr){6-7}\cmidrule(lr){8-9}
\textbf{Method} & Acc. & 95\% CI & Acc. & 95\% CI & Acc. & 95\% CI & Acc. & 95\% CI \\
\midrule
Unsteered
  & 0.614 & {[0.572, 0.656]}
  & 0.899 & {[0.887, 0.911]}
  & 0.560 & {[0.482, 0.634]}
  & 0.587 & {[0.543, 0.635]} \\
\midrule
ITI
  & 0.646 & {[0.604, 0.688]}
  & \textbf{0.913} & {[0.897, 0.927]}
  & 0.463 & {[0.384, 0.537]}
  & 0.557 & {[0.511, 0.604]} \\
AS
  & 0.566 & {[0.522, 0.610]}
  & 0.900 & {[0.883, 0.915]}
  & 0.518 & {[0.445, 0.592]}
  & 0.548 & {[0.501, 0.595]} \\
CD
  & 0.604 & {[0.562, 0.646]}
  & 0.874 & {[0.856, 0.892]}
  & 0.433 & {[0.360, 0.506]}
  & 0.593 & {[0.546, 0.639]} \\
\midrule
\textbf{\ours{}}
  & \textbf{0.648} & {[0.606, 0.690]}
  & \textbf{0.913} & {[0.898, 0.928]}
  & \textbf{0.604} & {[0.524, 0.677]}
  & \textbf{0.604} & {[0.557, 0.649]} \\
\bottomrule
\end{tabular}}
\end{table}

\begin{table}[h]
\centering
\caption{Bootstrap 95\% CIs on Llama-3.1-8B-Instruct across all benchmarks
         ($B{=}10{,}000$ resamples, percentile method).}
\label{tab:bootstrap_llama}
\resizebox{\textwidth}{!}{%
\begin{tabular}{l cc cc cc cc}
\toprule
 & \multicolumn{2}{c}{\textbf{MATH-500} ($N{=}500$)}
 & \multicolumn{2}{c}{\textbf{GSM8k} ($N{=}1319$)}
 & \multicolumn{2}{c}{\textbf{HumanEval} ($N{=}164$)}
 & \multicolumn{2}{c}{\textbf{MBPP} ($N{=}427$)} \\
\cmidrule(lr){2-3}\cmidrule(lr){4-5}\cmidrule(lr){6-7}\cmidrule(lr){8-9}
\textbf{Method} & Acc. & 95\% CI & Acc. & 95\% CI & Acc. & 95\% CI & Acc. & 95\% CI \\
\midrule
Unsteered
  & 0.478 & {[0.428, 0.516]}
  & 0.860 & {[0.841, 0.879]}
  & 0.561 & {[0.494, 0.640]}
  & 0.562 & {[0.515, 0.609]} \\
\midrule
ITI
  & 0.498 & {[0.456, 0.540]}
  & 0.855 & {[0.836, 0.874]}
  & 0.573 & {[0.494, 0.646]}
  & 0.548 & {[0.501, 0.595]} \\
AS
  & 0.506 & {[0.462, 0.550]}
  & 0.857 & {[0.839, 0.876]}
  & 0.591 & {[0.512, 0.665]}
  & 0.546 & {[0.499, 0.593]} \\
CD
  & 0.492 & {[0.456, 0.542]}
  & 0.858 & {[0.839, 0.876]}
  & 0.591 & {[0.518, 0.665]}
  & 0.555 & {[0.508, 0.602]} \\
\midrule
\textbf{\ours{}}
  & \textbf{0.530} & {[0.486, 0.574]}
  & \textbf{0.867} & {[0.848, 0.886]}
  & \textbf{0.604} & {[0.524, 0.677]}
  & \textbf{0.574} & {[0.527, 0.621]} \\
\bottomrule
\end{tabular}}
\end{table}


\end{document}